\documentclass[letterpaper]{article} 
\usepackage[preprint]{aaai2027}  
\usepackage[hyphens]{url}  
\usepackage{graphicx} 
\usepackage{natbib}  
\usepackage{caption} 
\usepackage{algorithm}
\usepackage{algorithmic}

\usepackage{newfloat}
\usepackage{listings}
\DeclareCaptionStyle{ruled}{labelfont=normalfont,labelsep=colon,strut=off} 
\floatstyle{ruled}
\newfloat{listing}{tb}{lst}{}
\floatname{listing}{Listing}

\usepackage{booktabs}

\usepackage{amsmath}
\usepackage{amsthm}
\usepackage{amssymb}
\usepackage{multirow}
\usepackage{rotating}
\usepackage{pifont}
\usepackage{dsfont}
\usepackage{colortbl}
\usepackage{makecell}

\usepackage{tikz}

\usepackage{subcaption}

\def\B{{\boldsymbol B}}

\def\G{{\boldsymbol G}}
\def\g{{\boldsymbol g}}

\def\S{{\boldsymbol S}}

\def\x{{\boldsymbol x}}

\def\y{{\boldsymbol y}}

\def\z{{\boldsymbol z}}

\def\BM{{\mathcal B}}

\def\DM{{\mathcal D}}

\def\FM{{\mathcal F}}

\def\LM{{\mathcal L}}

\newcommand{\tikzcmark}{%
\tikz[scale=0.23] {
    \draw[line width=1,line cap=round] (0.25,0) to [bend left=10] (1,1);
    \draw[line width=1,line cap=round] (0,0.35) to [bend right=1] (0.23,0);
}}

\newcommand{\tikzxmark}{%
\tikz[scale=0.23] {
    \draw[line width=1,line cap=round] (0,0) to [bend left=6] (1,1);
    \draw[line width=1,line cap=round] (0.2,0.95) to [bend right=3] (0.8,0.05);
}}

\newcommand{\std}[1]{{\scriptsize$\pm$#1}}
\newcommand{\unimodal}[1]{\cellcolor{gray!25}#1}

\newtheorem{theorem}{Theorem}

\newtheorem{lemma}{Lemma}

\title{Learning Through Game: Skewed Transfer of Tabular Knowledge to Strengthen Image Model}
\author{\textbf{Longfei Huang}$^{1}$\ \
\textbf{Shangdong Yang}$^{2}$\ \
\textbf{Yang Yang}$^{1}$\thanks{Corresponding author.}
}
\affiliations{
    $^{1}$Nanjing University of Science and Technology\\
    $^{2}$Nanjing University of Posts and Telecommunications
}

\begin{document}

\maketitle

\begin{abstract}
Multimodal tabular-image learning is gaining growing attention, yet it faces challenges due to tabular data unavailable at test time. A practical solution involves transferring tabular knowledge to images during training to enhance the performance of image models at inference. However, the overlooked yet important challenges lie in the modality imbalance between images and tables, as well as their asymmetric modality relationship in cross-modal transfer, which limits the auxiliary role of tabular data. To address these issues, we propose Skewed Knowledge Transfer (SKT), which asymmetrically transfers tabular knowledge to improve the image model by adaptive integration of modality gradients in a shared parameter space. Specifically, we first introduce a multimodal shared head, which allows the model to benefit from cross-modal structure without adding additional parameters. We then design a two-step Nash Bargaining strategy to effectively leverage tabular gradients. In the first step, SKT seeks a point of modality balance and uses preference awareness in the second step to steer combined gradients toward image-beneficial directions. Furthermore, we theoretically analyze the Pareto improvement and convergence of SKT. To this end, tabular knowledge is explicitly transferred to enhance image models. Empirical experiments on widely used tabular-image datasets reveal that SKT consistently improves image unimodal performance by using tabular data as auxiliary information.
\end{abstract}

\section{Introduction}
Tabular-image learning \cite{MML:conf/miccai/ZhengLZPXZJW22,MML:conf/miccai/XuZHZM16,TIL:conf/miccai/GrzeszczykKASTS24,TIP:conf/eccv/DuZWBOQ24,TIL:conf/aaai/FuZZCJ26,TIL:journals/corr/abs-2602-20223} aims to leverage structured tabular data to improve the holistic understanding of visual tasks, which has gained increasing attention in various fields, such as healthcare \cite{Healthcare:conf/kdd/ZhangCMZWWZ22,Healthcare:acosta2022multimodal,TIL:conf/wacv/HasnyFBS26} and marketing \cite{Market:conf/www/HeM16,DVM:conf/bigdataconf/HuangCLYO22}, and so on. Most existing methods adopt multimodal fusion to integrate information from tabular data and images. Despite numerous achievements, these methods assume that all modalities are available at test time, whereas obtaining real-world tabular data is often more challenging and costly \cite{MMCL:conf/cvpr/HagerMR23,CHARMS:conf/icml/JiangYW00Z24,TIL:journals/corr/abs-2512-19602}. For example, in healthcare, obtaining medical images depends on equipment, whereas deriving detailed and accurate diagnoses demands expert knowledge and is often more expensive and challenging \cite{TIL:journals/corr/abs-2510-15208}.

A practical solution is to leverage both modalities during training to facilitate the transfer of tabular knowledge to the image model, thereby enhancing its inference performance \cite{MMCL:conf/cvpr/HagerMR23,CHARMS:conf/icml/JiangYW00Z24,DBLP:journals/corr/abs-2510-08492,DBLP:conf/iclr/LeeY25a,DBLP:journals/corr/abs-2604-01579}. Although the main idea is to incorporate the tabular data as auxiliary information when training the image model, tabular-image transfer is hindered by critical  barriers: 1) Existing transfer methods overlook the interference caused by modality imbalance \cite{OGM:conf/cvpr/PengWD0H22,DBLP:conf/aaai/WangCZCZ26,DBLP:conf/aaai/GaoCGZ26} in multimodal learning, which limits the potential of unimodal models. 2) The canonical methods rely on cross-modal alignment to transfer knowledge. This strategy assigns equal importance to all modalities and is not specifically designed for the directed transfer of knowledge from one modality to another, which may lead the model to learn unstable intermediate states and hinder the target modality performance. As shown in Figure \ref{fig:intro} (a), the method CHARMS \cite{CHARMS:conf/icml/JiangYW00Z24} exhibits significant fluctuations in image loss, and its loss landscape visualization \cite{Land:conf/nips/Li0TSG18,Land:conf/iclr/ForetKMN21} is sharp minima, which is associated with suboptimal unimodel model quality. Therefore, a new cross-modal transfer method that mitigates modality imbalance and is tailored to directed transfer is urgently needed.

\begin{figure*}[t]
\centering
\begin{subfigure}{.49\linewidth}
\centering
\begin{subfigure}{.49\linewidth}
\centering
\includegraphics[width=\linewidth]{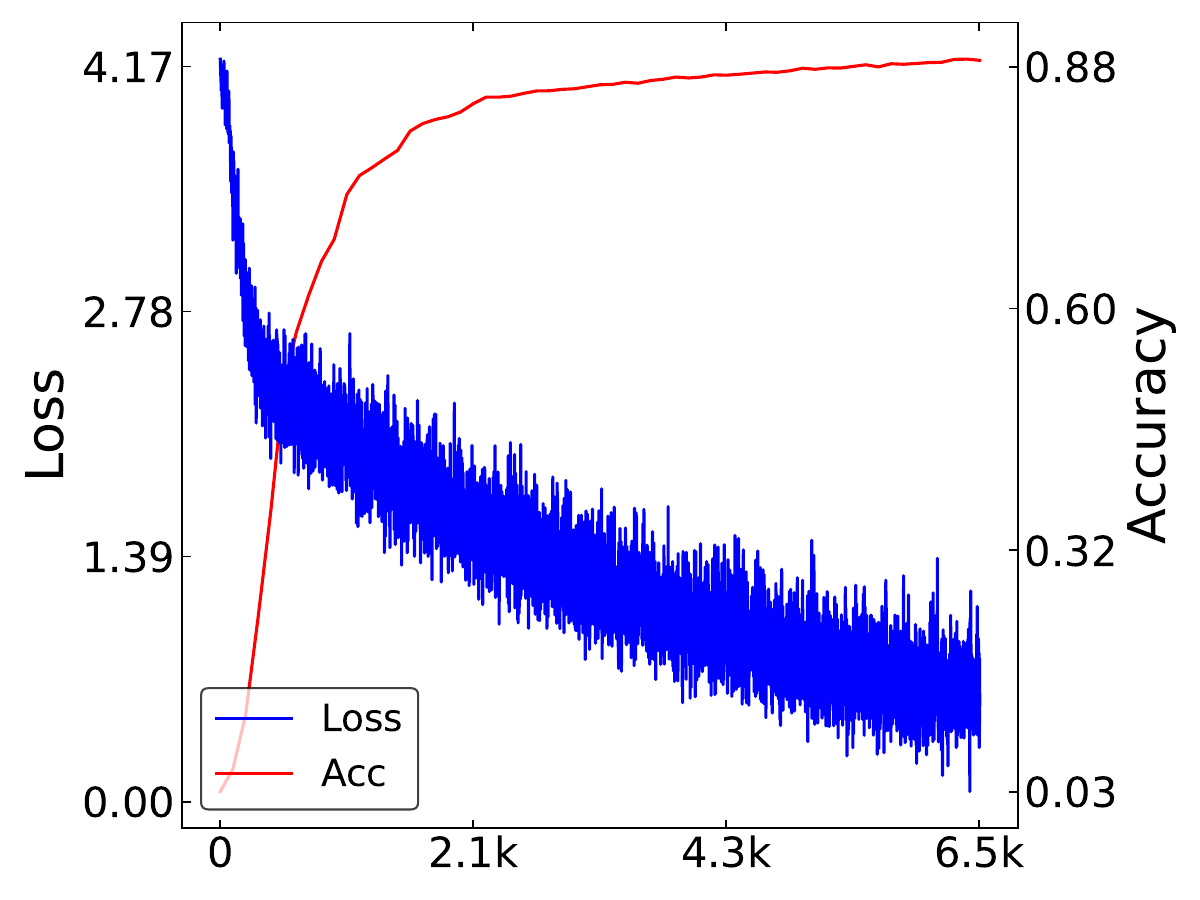}
\end{subfigure}
\hfill
\begin{subfigure}{.49\linewidth}
\centering
\includegraphics[width=\linewidth]{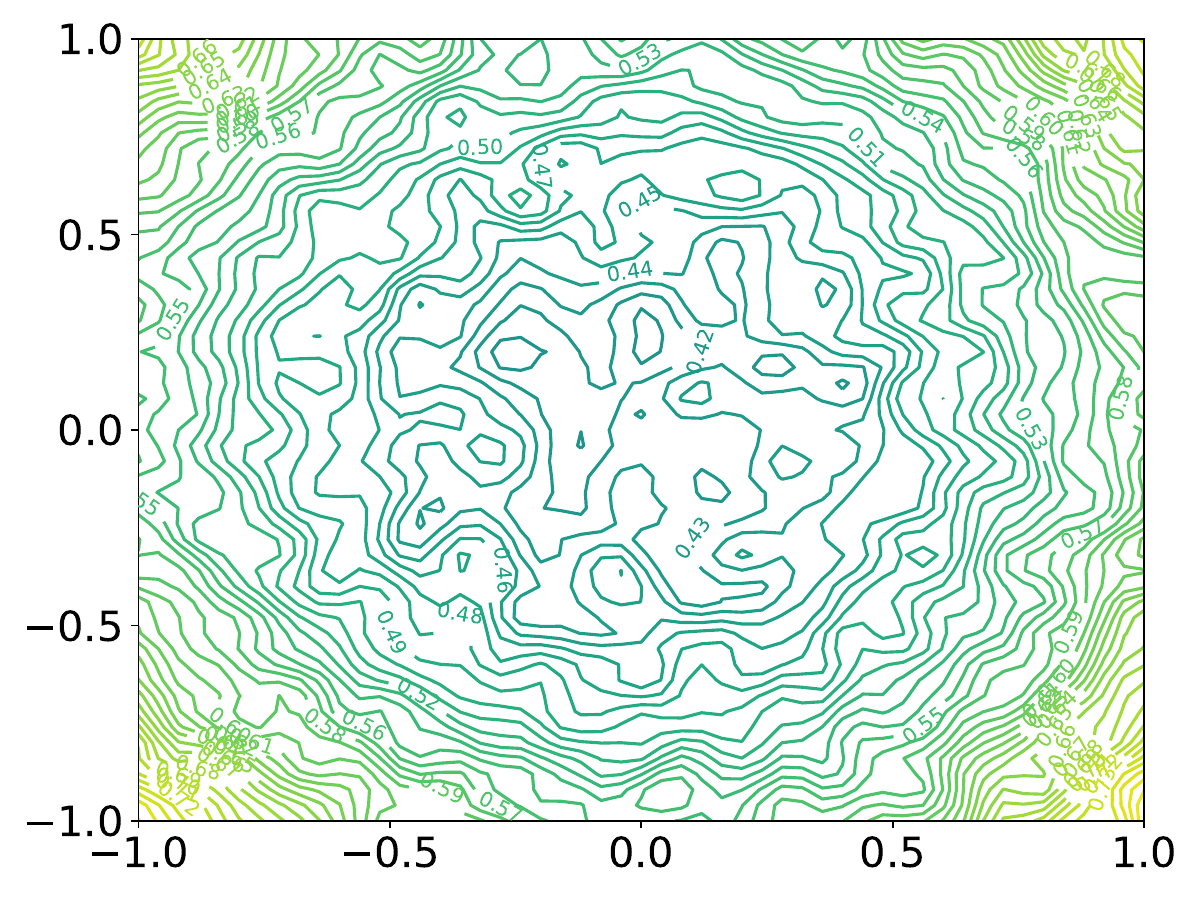}
\end{subfigure}
\caption{Loss, accuracy and loss landscape using CHARMS}
\label{fig:charms}
\end{subfigure}
\hfill
\begin{subfigure}{.49\linewidth}
\centering
\begin{subfigure}{.49\linewidth}
\centering
\includegraphics[width=\linewidth]{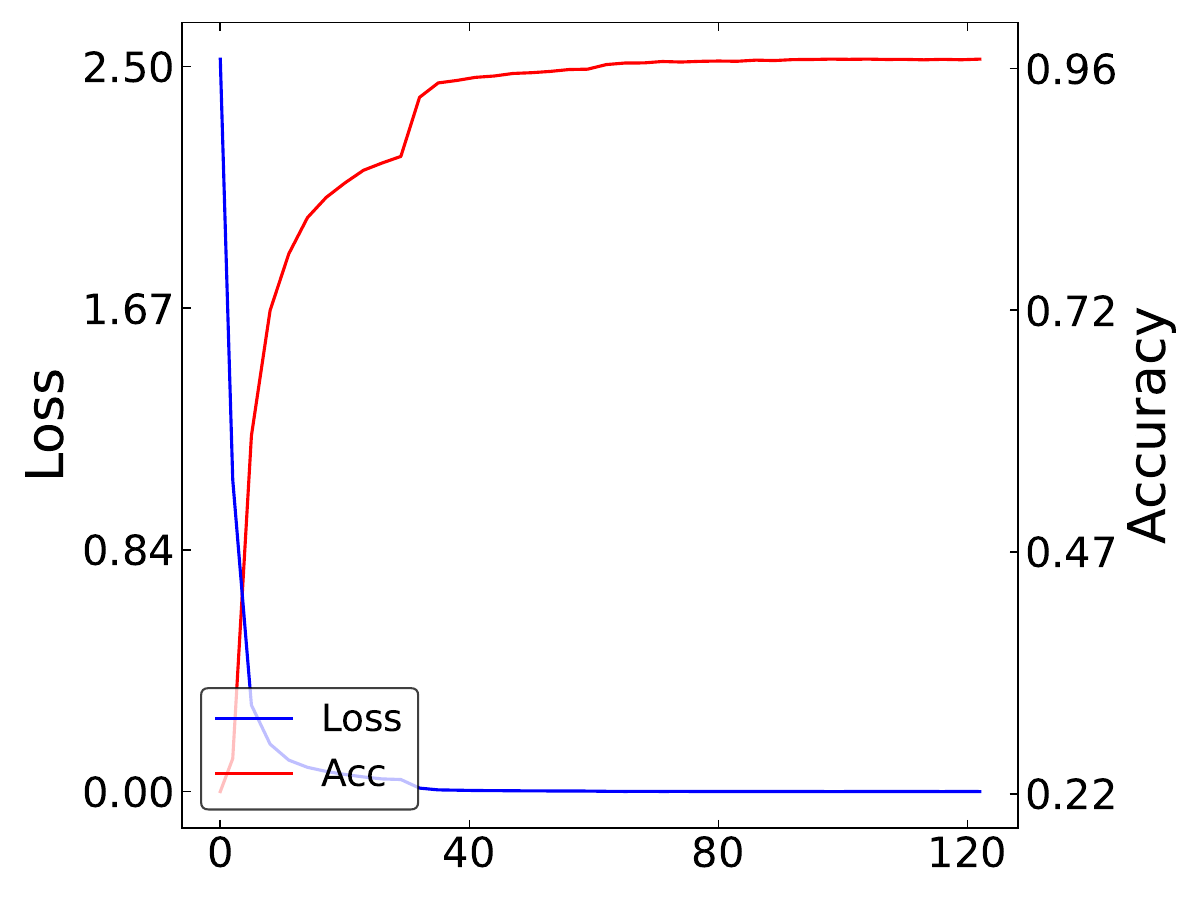}
\end{subfigure}
\hfill
\begin{subfigure}{.49\linewidth}
\centering
\includegraphics[width=\linewidth]{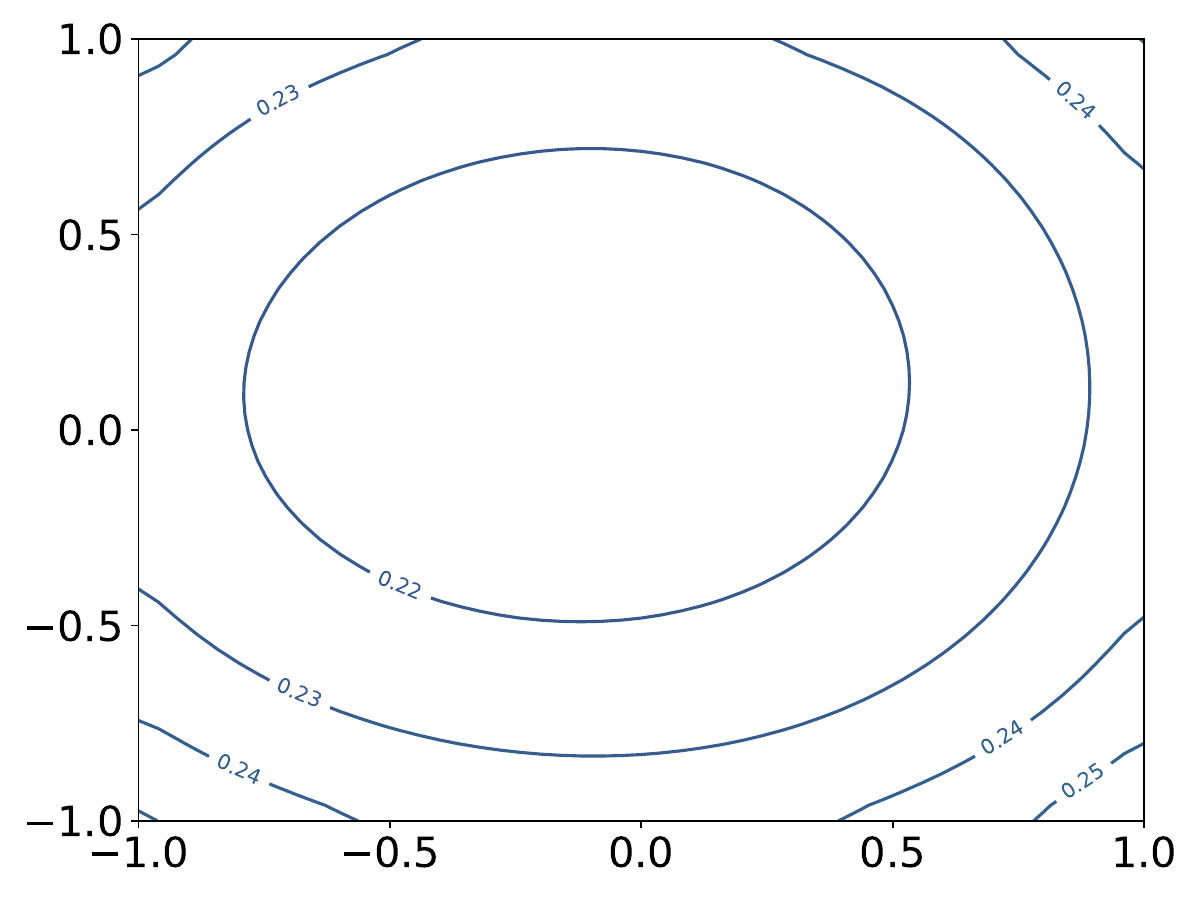}
\end{subfigure}
\caption{Loss, accuracy and loss landscape using SKT}
\label{fig:ours}
\end{subfigure}
\caption{Comparison between CHARMS and our approach SKT. All experiments are conducted on DVM dataset. (a) CHARMS exhibits oscillatory convergence in image task loss, and its loss landscape is noticeably sharp minima. (b) SKT achieves a stable loss decrease and improved performance with fewer iteration steps. Moreover, the unimodal model exhibits flat minima in the loss landscape, which is often associated with a high-quality model.}
\label{fig:intro}
\end{figure*}

To address these needs, we propose a Skewed Knowledge Transfer (SKT) approach, which extracts cross-modal gradient information in a shared parameter space. It establishes an asymmetric modality relationship where images play the dominant role and tables serve as auxiliary information. Our main idea is: expand the pie, then divide it. Therefore, SKT first alleviates modality imbalance and then adaptively uses tabular gradient knowledge to assist image learning.

Specifically, we first introduce a multimodal shared head as a bridge for cross-modal transfer. This enables the unimodal model to leverage cross-modal structure without introducing additional parameters, thereby effectively verifying that the gain indeed comes from the incorporation of cross-modal information. We then design a two-step Nash Bargaining strategy to adaptively leverage cross-modal gradient knowledge. Unlike the standard Nash Bargaining Solution, which treats players equally in symmetric settings, we extend it to the multimodal setting and further generalize the relation among players to an asymmetric case by two-step bargaining. In the first step, which aims to alleviate modality imbalance, SKT seeks a sharp multimodal combined gradient that benefits all players and yields higher overall gain. In the second step, which realizes directed cross-modal transfer, SKT further adjusts the image weight by a preference-aware adjustment, so that the final multimodal combined gradient is actually biased toward the image model. In this way, SKT transfers tabular knowledge to enhance the image unimodal model. In addition, we theoretically analyze multimodal loss convergence and Pareto improvement under the proposed SKT framework. In Figure \ref{fig:intro} (b), we present the image unimodal performance, loss, and loss landscape of SKT. We can find the image model in SKT converges stably and exhibits a flat minimum loss landscape, achieving higher accuracy and differing markedly from the equal alignment method (CHARMS). Our contributions are outlined as follows:
\begin{itemize}
\item We propose Skewed Knowledge Transfer (SKT), which considers modality imbalance and the asymmetric modality relationship in cross-modal transfer. SKT enhances the image model by adaptively integrating gradients, addressing the challenge of test-time tabular data missing.
\item We theoretically analyze the convergence of SKT and show that it achieves image-biased Pareto improvement.
\item Experiments reveal that SKT effectively leverages tabular knowledge as auxiliary information and consistently improves the performance of the image unimodal model.
\end{itemize}

\section{Related Work}
\subsection{Multimodal Learning}
Data in machine learning applications often comes in multiple modalities, including images, text, video, audio, tables, etc., providing diverse and rich sources of information. The goal of multimodal learning \cite{KMG:conf/aaai/FuSZY24,MML:conf/nips/Jiang2025aug,MMDL:conf/icml/NgiamKKNLN11,MMDL:conf/icml/JiaYXCPPLSLD21,MML:conf/aaai/FangZC26} is to integrate heterogeneous information from diverse sensors. Recent studies \cite{OGR-GB:conf/cvpr/WangTF20,MLA:conf/cvpr/ZhangYBY24} show that, due to modality heterogeneity, multimodal learning suffers from modality imbalance, which limits unimodal potential and overall performance. Many works attempt to address this issue by balancing the optimization of different modalities \cite{MMPareto:conf/icml/WeiH24,MML:journals/corr/abs-2602-13015}. Different from fusion, transfer trains on multimodal data but infers on a single modality, which has significant practical implications for applications.

\subsection{Transfer Knowledge from Table to Images}
An emerging multimodal learning example is tabular-image learning, which has received increasing attention\cite{TIL:conf/miccai/PolsterlWW21,TIL:journals/artmed/BorsosAH24,TIL:conf/nips/JiangXSLCLZZY25,TIL:conf/aaai/FuZZCJ26}. Most works \cite{DAFT:journals/neuroimage/WolfPWIa22,TIL:journals/corr/abs-2602-20223} follow the traditional fusion methods, without considering the test-time tabular data unavailable challenge. Given that tabular data are costly to acquire in real-world settings, access to all modalities is often infeasible. A practical solution is to leverage multimodal information to build a robust unimodal model. Unlike another line of methods \cite{MML:conf/kdd/CaiWGSJ18,MML:journals/pami/PanLXS22} that imputes missing modalities, this cross-modal transfer paradigm avoids the noise introduced by imputation and is better suited to computation-constrained scenarios \cite{DBLP:conf/kdd/WangZTZ20}. A representative transfer strategy is knowledge distillation (KD) \cite{KD:journals/corr/HintonVD15}, and other methods rely on cross-modal alignment learning. For example, MMCL \cite{MMCL:conf/cvpr/HagerMR23} adopts a self-supervised contrastive learning framework that leverages both image and tabular data to train unimodal encoders, followed by supervised fine-tuning on image data. CHARMS \cite{CHARMS:conf/icml/JiangYW00Z24} employs channel tabular alignment and optimal transport to selectively align tabular attributes with image channels during joint training. However, these methods overlook modality imbalance during training and treat all modalities equally, which is inconsistent with the goal of knowledge transfer that prioritizes a target modality. In this work, we address these issues by balancing gradients and establishing an asymmetric relationship.
 
\subsection{Nash Bargaining}
Nash Bargaining \cite{nash1953two} is one of the most representative cooperative bargaining game models, determining a Pareto-optimal \cite{hochman1969pareto} outcome by maximizing the product of participants' utilities to ensure fairness in negotiation. In recent years, the Nash Bargaining Solution has been widely applied to fairness problems in machine learning to resolve gradient conflicts in different scenarios \cite{Nash:conf/icml/NavonSAMKCF22,Nash:conf/nips/0005YCF0J024,Nash:conf/iccv/Wu_2025_ICCV,Nash:conf/aaai/XuYYFCD26}. Although Nash Bargaining has the potential to facilitate the coordination and transfer of modality gradient information, the relationship among modalities in cross-modal transfer are asymmetric, making the Nash Bargaining Solution not directly applicable to tabular-image transfer tasks. In this paper, we introduce Nash Bargaining into the multimodal setting and adopt a two-stage Nash Bargaining strategy to adapt it to the cross-modal transfer setting.

\section{Skewed Knowledge Transfer}
\subsection{Preliminaries}
Assume that we have $N$ data points, each of which has image and tabular modalities. Without loss of generality, we use $\DM=\{\x_i^I,\x_i^T, y_i\}_{i=1}^N$ to denote a training dataset, where $\x_i^I$ and $\x_i^T$ denote the $i$-th data point of image and tabular description, respectively. In addition, we assume there are $Y$ classes in total, and $y_i \in \left[ Y \right]=\left\{ 1,...,Y \right\}$. 

For the sake of simplicity, we use superscript $r$ to indicate the module corresponding to a specific modality in this section, where $r\in\{I, T\}$. With the rapid growth of deep learning, representative approaches \cite{MMCL:conf/cvpr/HagerMR23,TIP:conf/eccv/DuZWBOQ24,CHARMS:conf/icml/JiangYW00Z24,STiL:conf/cvpr/DuLOQ25} have adopted deep neural network (DNN) for tabular-image learning. Following these methods, we also utilize DNN to construct our models $f^r$. Specifically, we use $\phi^r(\cdot)$ to denote feature extractors and leverage tabular data through multimodal joint training. Then the features can be calculated by $\z^r=\phi^r(\x^r;\theta^r)$, where $\theta^r$ denotes the extractor parameters.

In scenarios where expert knowledge is unavailable due to high acquisition costs, we expect image model $f^I$ to produce still stable and accurate predictions given only image data $\x^I$. Therefore, in the testing phase, we evaluate the performance of $f^I$ by assessing its prediction accuracy on test image $\x^I_{i}$. Due to asymmetric deployment where only the image model is used at inference, our learning objective is inherently directional rather than symmetric.

\subsection{Multimodal Framework}
As model parameters contain learned knowledge, knowledge transfer can be conducted from parameter level. Introducing additional shared parameters will change the capacity of the model itself, making it difficult to accurately assess the actual gains brought by cross-modal information. Hence, we use a shared task head $\psi$ as a bridge for cross-modal interaction, constructing a shared parameter space across heterogeneous models without introducing extra parameters. Therefore, the prediction of given data can be calculated by $\psi(\cdot)$: $f^r(\x^r)=\psi(\z^r;\Theta)$, where $\Theta$ denotes the shared head parameters. The corresponding multimodal objective can be formed as:
\begin{equation}
\begin{aligned}
\LM \left ( \x^I, \x^T, \y \right )&=\LM^I \left ( \x^I, \y \right ) + \LM^T \left ( \x^T, \y \right ) \\
&=\sum_{i=1}^{N} \ell (f^I(\x^I_{i}),y_i)+\ell (f^T(\x^T_{i}),y_i),\label{obj:mml}
\end{aligned}
\end{equation}
where $\LM^I$ and $\LM^T$ denote the losses for image and tabular modalities, respectively. The $\ell$ denotes the cross-entropy loss for classification tasks or mean squared error (MSE) for regression tasks. 

Since both images and tabular data are projections of the same underlying reality, forcing them through shared weights can extract synergies by accumulating training gradients on the same parameters. We can use game theory to coordinate unimodal utilities and overall gain in multimodal learning.

\noindent\textbf{Multimodal Objective.} We bring game theory into the multimodal setting. In multimodal learning, we have two players, i.e., image and tabular, aiming to offer gradients to maximize the overall multimodal gain. We obtain unimodal gradients $\g_r$ for the shared head via unimodal loss $\LM^r$.

Let $\G=\left [\g_I, \g_T \right ]$. Our objective is to construct a combined gradient $\boldsymbol{\tilde{g}}=\G\boldsymbol{\alpha}$ using non-negative weights $\boldsymbol{\alpha}=\left [{\alpha}_I, {\alpha}_T \right ]^{\top}$ to minimize the multimodal loss. Inspired by \cite{Nash:conf/icml/NavonSAMKCF22,Nash:conf/nips/0005YCF0J024,Nash:conf/iccv/Wu_2025_ICCV}, we define the utility function $u^r(\boldsymbol{\tilde{g}})$ as
\begin{align}
u^r(\boldsymbol{\tilde{g}}) =\g_{r}^{\top} \boldsymbol{\tilde{g}}=(\G^{\top}\G\boldsymbol{\alpha})_r.\label{obj:uf}
\end{align}
The intuition is that the utility tells us how much of the proposed update is applied in modality $r$. If $\g_r$ and $\boldsymbol{\tilde{g}}$ align well, updating the shared head along $\boldsymbol{\tilde{g}}$ leads to a larger decrease in $\LM^r$, and more effectively minimizing multimodal loss in Equation \ref{obj:mml}. Our main idea is to first make the pie bigger and then divide it. Therefore, our multimodal objective is to determine a combined gradient $\boldsymbol{\tilde{g}}$ that improves loss in Equation \ref{obj:mml}, thereby achieving higher overall gain. Denote $\B_{\epsilon}$ the ball of radius $\epsilon$ centered at 0, the objective can be formalized as:
\begin{align}
\max_{\boldsymbol{\tilde{g}} \in \B_{\epsilon}} \sum_{r\in\{I,T\}}\log({u^r(\boldsymbol{\tilde{g}})}),\label{obj:object}
\end{align}
where the logarithm is adopted to help balance and align with the property that utility gains less benefit as it continues to improve. With this objective, we can mitigate modality imbalance and exploit the potential of each modality to maximize multimodal gain. To achieve this objective while favoring the image model, we design a two-step Nash Bargaining strategy to regulate the weights $\boldsymbol{\alpha}$.

\subsection{Two-step Nash Bargaining}
\noindent\textbf{Balancing Multimodal Learning.} The first step aims to address modality imbalance in joint learning, thereby unleashing unimodal potential and improving overall gain, i.e., expanding the pie. Concretely, we use the Nash Bargaining Solution to derive a set of Pareto-optimal weights for balancing modality-specific gradients. We present the Nash Bargaining Solution to Equation \ref{obj:object} by following two theorems. We provide proofs in the appendix.

\begin{theorem}[Optimality condition]\label{thm:optimality}
Let some scalar $\mu>0$ and $L(\boldsymbol{\tilde{g}})=\sum_{r\in\{I,T\}}\log({u^r(\boldsymbol{\tilde{g}})})$. The optimal solution $\boldsymbol{\tilde{g}}^{*}$ to Equation \ref{obj:object} must satisfy:
\begin{align}
\nabla L(\boldsymbol{\tilde{g}}^{*})=\mu \boldsymbol{\tilde{g}}^{*}, \boldsymbol{\tilde{g}}^{*}=\G\boldsymbol{\alpha}^{*}. \label{thm:opt}
\end{align} 
\end{theorem}

\begin{theorem}[Solution characterization]\label{thm:solution}
The solution to Equation \ref{thm:opt}, up to scaling, is $\boldsymbol{\tilde{g}}^{*}=\G\boldsymbol{\alpha}^{*}$ where $\boldsymbol{\alpha}^{*}$ is the solution to
\begin{align}
\G^{\top}\G\boldsymbol{\alpha}=\frac{1}{\boldsymbol{\alpha}}, \label{thm:obj}
\end{align}
with the element-wise reciprocal $\tfrac{1}{\boldsymbol{\alpha}}$.
\end{theorem}

Hence, we can solve $\boldsymbol{\alpha}$ in Equation \ref{thm:obj} to obtain the solution to Equation \ref{obj:object}. Following \cite{Nash:conf/icml/NavonSAMKCF22}, we adopt a general approximation $\boldsymbol{\alpha}$ framework that iteratively improves the solution for $\boldsymbol{\alpha}$ via a first-order approximation. Denote $\S=\G^{\top}\G$, then our optimization problem is formulated as:
\begin{equation}
\begin{aligned}
\boldsymbol{\alpha}^{(t)} =&\arg\min_{\boldsymbol{\alpha}} \sum_{r\in \left \{ I,T \right \} } (\S\boldsymbol{\alpha})_r+\gamma (\boldsymbol{\alpha};\boldsymbol{\alpha}^{(t-1)})\\
s.t.\;& -\log(\alpha_r)-\log((\S\boldsymbol{\alpha})_r)\le 0, r\in \left \{ I,T \right \} \\
&\boldsymbol{\alpha} > 0,\label{opt:balance}
\end{aligned}
\end{equation}
where $\gamma (\boldsymbol{\alpha};\boldsymbol{\alpha}^{(t-1)})$ is a first-order linear regularization term. The formulation is as follows:
\begin{align}
\gamma(\boldsymbol{\alpha})=\sum_{r\in \left \{ I,T \right \} }\log(\alpha_r)+\sum_{r\in \left \{ I,T \right \} }\log((\S\boldsymbol{\alpha})_r), \\
\gamma (\boldsymbol{\alpha};\boldsymbol{\alpha}^{(t-1)})\approx \nabla \gamma(\boldsymbol{\alpha}^{t-1})^{\top}(\boldsymbol{\alpha}-\boldsymbol{\alpha}^{(t-1)}).
\end{align}
This regularization simplifies computation while preventing $\alpha$ from oscillating or collapsing to zero by capturing historical information. 

In the first step, we solve a problem to derive Pareto-optimal weights $\boldsymbol{\alpha}^{(t)}$ for balancing multimodal learning. We then design a second step Nash bargaining strategy tailored to the asymmetric learning objective in cross-modal transfer.

\noindent\textbf{Asymmetric Cross-modal Transfer.} The goal of cross-modal transfer is to transfer knowledge from one modality to another, naturally giving rise to an asymmetric modality relationship. Therefore, the second step introduces modality preference value and adaptively adjusts the target modality weight, i.e., allocating more of the pie to the target modality. Specifically, we keep the tabular weight solution $\alpha_{T}^{(t)}$ fixed and optimize only the image weight $\alpha_I$ based on Equation \ref{opt:balance} solution. We can modify Equation \ref{opt:balance} to ensure the priority of the images by relaxing the constraint. Therefore, the optimization problem can be formulated as:
\begin{equation}
\begin{aligned}
\boldsymbol{\alpha}^{(t+1)} &=\arg\min_{\alpha_I} (\S\boldsymbol{\alpha})_I+\gamma (\boldsymbol{\alpha};\boldsymbol{\alpha}^{(t)})\\
s.t.\;& -\log(\alpha_r)-\log((\S\boldsymbol{\alpha})_r)+\log(p)\le 0, r\in \left \{ I,T \right \} \\
&\alpha_I > 0,\label{opt:transfer}
\end{aligned}
\end{equation}
where $p$ is a constant that serves as a modality preference value to adjust the transfer direction. When $p>1$, the lower bound of the image weight is increased, allowing the combined gradient to favor the image unimodal model while maintaining overall Pareto optimality.

Since the tabular gradient weight $\alpha_T$ is fixed, the linear regularization term $\gamma (\boldsymbol{\alpha};\boldsymbol{\alpha}^{(t)})$ applies only to $\alpha_I$. Considering only the optimizable variable, with the term corresponding to $\alpha_T$ set to 0, the formulation can be written as:
\begin{align}
\gamma (\boldsymbol{\alpha};\boldsymbol{\alpha}^{(t)})\approx \nabla_{\alpha_I} \gamma(\boldsymbol{\alpha}^{(t)})^{\top}(\alpha_I-\alpha^{(t)}_I).
\end{align}
It serves to prevent $\alpha_I$ from oscillating or converging to an extreme value. 

In the second step, we enable directed transfer of tabular knowledge to the image by introducing modality preference and solving only for the image weight. Meanwhile, the fixed tabular weight, together with the Pareto-optimal constraint, prevents the tabular gradient from being overly suppressed.

Our algorithm is summarized in Algorithm \ref{algo:skt}. To sum up, SKT employs a multimodal shared head for cross-modal transfer. Subsequently, SKT addresses modality imbalance and the asymmetric modality relationship in cross-modal transfer via a two-step Nash Bargaining strategy.

\begin{algorithm}[t]
\caption{Learning algorithm of our proposed method.}\label{algo:skt}
\begin{algorithmic}[1]
\STATE {\bfseries Input:}Training dataset $\DM=\{\x_i^I,\x_i^T, y_i\}_{i=1}^N$.
\STATE {\bfseries Output:}Parameters$\{\theta^I,\Theta\}$ for the image model.\\
\textbf{INIT} Initialize modality weight $\boldsymbol{\alpha}=\left [{\alpha}_I, {\alpha}_T \right ]^{\top}$. Initialize iteration $k=1$. Initialize encoder parameters $\{\theta^I,\theta^T\}$ and shared parameters $\Theta$.
\FOR{$k=1\mapsto\#iterations$}
\STATE Sample a mini-batch $\BM$ from $\DM$;
\STATE Feed forward the batched data $\BM$ to the model;
\STATE Calculate multimodal loss $\LM \left ( \x^I, \x^T, \y \right )$;
\STATE Calculate gradient using back-propagation;
\STATE \textbf{// Skewed Knowledge Transfer}
\STATE Obtain $\g_I$ and $\g_T$ for shared head;
\STATE Set $\G=\left [ \g_I, \g_T \right ]$ and $\S=\G^{\top}\G$;
\STATE Solve Equation \ref{opt:balance} to obtain $\boldsymbol{\alpha}$;
\STATE \textbf{// Balancing Multimodal Learning}
\STATE Fix $\alpha_T$ and solve Equation \ref{opt:transfer} to obtain $\boldsymbol{\alpha}^{(k)}$;
\STATE \textbf{// Asymmetric Cross-modal Transfer}
\STATE Calculate combined gradient: $\boldsymbol{\tilde{g}} = \G\boldsymbol{\alpha}^{(k)}$;
\STATE Update $\Theta$ with combined gradient $\boldsymbol{\tilde{g}}$;
\STATE Update $\{\theta^I,\theta^T\}$ with gradients $\{\nabla_{\theta^I}\LM^I,\nabla_{\theta^T}\LM^T\}$;
\ENDFOR
\STATE \textbf{Return} $\{\theta^I,\Theta\}$.
\end{algorithmic}
\end{algorithm}

\subsection{Theoretical Properties}
SKT transfers tabular knowledge to the image model in a shared parameter space by a two-step Nash bargaining strategy. In this section, we analyze the key theoretical properties of this strategy.

\begin{lemma}[Boundedness]\label{lemma:Boundedness}
If $\g_r$ is $\sigma$-bounded for $r \in \{ I, T\}$, then $\frac{1}{\sqrt{2}\sigma} \leq \left \| \alpha_T \right \|$ and $\frac{\sqrt{p}}{\sqrt{2}\sigma} \leq \left \| \alpha_I \right \|$, where $\sigma < \infty$ and preference $p>1$.
\end{lemma}

\begin{theorem}[Pareto improvement]\label{thm:Pareto}
Assume $\LM^r$ is differential and Lipschitz-smooth with constant $C>0$. If the learning rate at step $k$ is set to $\eta^{\left ( k\right )}\leq\frac{1}{C\alpha_{r}^{\left ( k\right )}}$, then the update ensures $\LM^r\left ( \Theta^{\left ( k+1\right )}\right ) \leq \LM^r \left ( \Theta^{\left ( k\right )}\right )$ for both players.
\end{theorem}

\begin{theorem}[Convergence]\label{thm:Convergence}
Assume unimodal loss $\LM^r\left ( \Theta^{\left ( k\right )}\right )$ is monotonically decreasing and bounded below. Then the multimodal loss $\LM\left ( \Theta\right )$ converges to $\LM\left ( \Theta^*\right )$ and $\Theta^*$ is the stationary point of $\LM\left ( \Theta\right )$.
\end{theorem}

These theoretical results show that SKT reduces loss in Equation \ref{obj:mml}, thereby addressing modality imbalance and improving the overall gain. Since the second step Nash Bargaining favors the image, the combined gradient is biased toward the image model and empirically improves its performance when used to update $\Theta$. We provide proofs in the appendix.

\section{Experiments}
\subsection{Experimental Settings} \label{exp:setup}
\noindent\textbf{Dataset:} We conduct experiments on six datasets: Data Visual Marketing (DVM) \cite{DVM:conf/bigdataconf/HuangCLYO22}, SUNAttribute \cite{SUN:journals/ijcv/PattersonXSH14}, CelebA \cite{CelebA:conf/iccv/LiuLWT15}, PetFinder-adoption (Adoption), PetFinder-pawpularity (Pawpularity), and Avito \cite{CHARMS:conf/icml/JiangYW00Z24}. 

\textbf{DVM} is built from 335,562 used car advertisements. The tabular data includes sales and technical specifications. Following previous work \cite{CHARMS:conf/icml/JiangYW00Z24}, car models with less than 700 samples were removed. \textbf{SUNAttribute} is constructed from 717 SUN dataset categories, where the "open area" attribute serves as the label. \textbf{CelebA} is a facial attribute dataset containing 202,599 face images of 10,177 celebrities, with the "Attractive" attribute used as the label. \textbf{PetFinder-adoption} dataset comes from a Kaggle competition where the task is to predict the speed at which a pet is adopted. \textbf{PetFinder-pawpularity} dataset also comes from a Kaggle competition where the task was to predict the popularity of a pet. \textbf{Avito} is a dataset from Russia's largest classified advertisements website. Except for the CHARMS method, which requires raw data, we standardize numerical tabular fields using z-score normalization with a mean value of 0 and standard deviation of 1. Details are provided in the appendix.

\noindent\textbf{Evaluation:} In addition to image learning, we also compare with several cross-modal transfer methods: knowledge distillation (KD) \cite{KD:journals/corr/HintonVD15}, modality focus hypothesis (MFH) \cite{MFH:conf/iclr/XueGRZ23}, fixed model reuse (FMR) \cite{FMR:conf/aaai/YangZFJZ17}, multimodal contrastive learning (MMCL) \cite{MMCL:conf/cvpr/HagerMR23} and channel tabular alignment with optimal transport (CHARMS) \cite{CHARMS:conf/icml/JiangYW00Z24}. We use accuracy to measure the performance for classification tasks. For the regression task, we use root mean square error (RMSE) for performance evaluation.

\begin{table*}[ht]
\centering
\begin{tabular}{lc c c c c c}
\toprule
\multirow{1}{*}{\textbf{Method}} & \multicolumn{1}{c}{\textbf{DVM} $\uparrow$} & \multicolumn{1}{c}{\textbf{SUN} $\uparrow$} & \multicolumn{1}{c}{\textbf{CelebA} $\uparrow$} & \multicolumn{1}{c}{\textbf{Adoption} $\uparrow$} & \multicolumn{1}{c}{\textbf{Pawpularity} $\downarrow$} & \multicolumn{1}{c}{\textbf{Avito} $\downarrow$} \\
\midrule
Unimodal & 0.8743 & 0.8361 & 0.8146 & 0.3477 & 18.6150 & 0.2512 \\
\midrule
KD     & \unimodal{0.8390} & 0.8382 & \unimodal{0.8118} & 0.3532 & \unimodal{19.0683} & 0.2499 \\
MFH    &   -    & \unimodal{0.8312} & \unimodal{0.7507} & \unimodal{0.3041} & \unimodal{43.1455} & \unimodal{0.2873} \\
FMR    & \unimodal{0.8427} & \unimodal{0.8347} & \unimodal{0.8003} & 0.3526 & \unimodal{19.3517} & \unimodal{0.2937} \\
MMCL   & \unimodal{0.8203} & 0.8431 & \unimodal{0.8041} & \unimodal{0.2981} &    -    &    -   \\
CHARMS & \underline{0.9175} & \underline{0.8661} & \underline{0.8220} & \underline{0.3603} & \underline{18.4314} & \underline{0.2495} \\
\midrule
Baseline    & \unimodal{0.8552} & {0.8438} & \unimodal{0.8006} & {0.3514} & \unimodal{19.1438} & \unimodal{0.2539} \\ \midrule
SKT    & \makecell{\bf0.9664 \\ \std{0.0012}} & \makecell{\bf0.8682 \\ \std{0.0011}} & \makecell{\bf0.8273 \\ \std{0.0014}} & \makecell{\bf0.3610 \\ \std{0.0043}} & \makecell{\bf18.2809 \\ \std{0.0578}} & \makecell{\bf0.2437 \\ \std{0.0007}} \\
\bottomrule
\end{tabular}
\caption{Comparisons with baseline methods on DVM, SUNAttribute, CelebA, Adoption, Pawpularity, and Avito datasets. The first four are classification tasks while the last two are regression tasks. The best and second-best results are highlighted in bold and underline, respectively. The gray background denotes the performance based on cross-modal transfer underperforming unimodal learning.}
\label{tab:main-exp}
\end{table*}

\noindent\textbf{Implementation Details:} Following the setting of MMCL \cite{MMCL:conf/cvpr/HagerMR23}, we adopt ResNet50 as the image encoder and MLPs as the tabular encoder for all datasets. The image encoder is initialized with weights pretrained on ImageNet-1k, while the tabular encoder is randomly initialized. The task head consists of a fully connected layer with a hidden size of 2048, which is shared across all modalities. In addition, the best model is selected from the validation set using early stopping. Specifically, the batch size is searched in \{64, 128, 256\}, the learning rate is searched in \{1e-4, 5e-4, 1e-3, 5e-3, 1e-2, 5e-2\} and the weight decay is 1e-4. The image modality preference $p$ is searched in \{1.1, 1.2, 1.25, 1.33, 1.5, 2.0\}. In the course of the experiment, we implement our method with PyTorch and conduct experiments with a single NVIDIA RTX A6000. More details are provided in the appendix.

\subsection{Main Results}
The results on all datasets are reported in Table \ref{tab:main-exp}, where “-” indicates that the corresponding method cannot handle complex classification or regression tasks. The first four datasets are for classification, while the last two are for regression. The best and second-best results are highlighted in bold and underline, respectively. Gray background denotes transfer results that underperform unimodal learning. ``Unimodal'' refers to image learning, and ``Baseline'' denotes the naive shared head strategy. From Table \ref{tab:main-exp}, we can draw the following observations: 1) Compared with unimodal learning, equality-based cross-modal transfer methods are not consistently able to effectively leverage tabular information to improve the image model, and may even introduce negative transfer. 2) The naive shared head strategy also fails to ensure effective knowledge transfer, as it treats all modalities equally and ignores modality imbalance. 3) Our method consistently improves over the unimodal baseline across all datasets and outperforms other transfer baselines.

\begin{table}[t]
\centering
\begin{tabular}{cc|c|c|c}
\toprule
\multirow{1}{*}{Balance} & \multirow{1}{*}{Transfer} & \multicolumn{1}{c|}{\textbf{DVM} $\uparrow$} & \multicolumn{1}{c|}{\textbf{SUN} $\uparrow$} & \multicolumn{1}{c}{\textbf{Avito} $\downarrow$} \\
\midrule
\tikzxmark & \tikzxmark & 0.8552 & 0.8438 & 0.2539 \\
\tikzcmark & \tikzxmark & 0.9416 & 0.8445 & 0.2461 \\
\tikzxmark & \tikzcmark & 0.9320 & 0.8417 & 0.2535 \\ 
\tikzcmark & \tikzcmark & \bf 0.9664 & \bf 0.8682 & \bf 0.2437  \\
\bottomrule
\end{tabular}
\caption{Results of ablation studies on DVM, SUNAttribute and Avito datasets.}
\label{tab:ablation}
\end{table}

\begin{table}[t]
\centering
\begin{tabular}{ccccc}
\toprule
\multirow{1}{*}{Method} & \multicolumn{1}{c}{\textbf{DVM}} & \multicolumn{1}{c}{\textbf{SUN}} & \multicolumn{1}{c}{\textbf{CelebA}} & \multicolumn{1}{c}{\textbf{Adoption}} \\
\midrule
Unimodal & 0.8743 & 0.8361 & 0.8146 & 0.3477 \\
w/ SAM & 0.8964 & 0.8466 & 0.8041 & 0.3346  \\
\midrule
SKT & \bf 0.9664 & \bf 0.8682 & \bf 0.8273 & \bf 0.3610  \\
\bottomrule
\end{tabular}
\caption{Comparison of the SKT and SAM strategy on classification tasks.}
\label{tab:sam}
\end{table}

\subsection{Ablation Study}
We investigate the effectiveness of SKT by analyzing the impact of the key components in two-step Nash Bargaining, i.e., Equation \ref{opt:balance} and \ref{opt:transfer}, which are denoted as balancing and asymmetry, respectively. The results on DVM, SUNAttribute and Avito datasets are reported in Table \ref{tab:ablation}. From the results of Table \ref{tab:ablation}, we can find that: 1) Compared with the baseline, balancing multimodal learning via Equation \ref{opt:balance} improves image performance. 2) Introducing asymmetry via Equation \ref{opt:transfer} also improves transfer performance, but the gain is smaller than that from balanced multimodal learning. This suggests that a one-step asymmetric bargaining is constrained by modality imbalance. 3) By considering modality imbalance and asymmetric modality relationships, SKT achieves the largest model performance improvement.

\begin{figure*}[t]
\centering
\begin{minipage}[t]{0.49\textwidth}
\centering
\begin{subfigure}[t]{.49\linewidth}
\centering
\includegraphics[width=\linewidth]{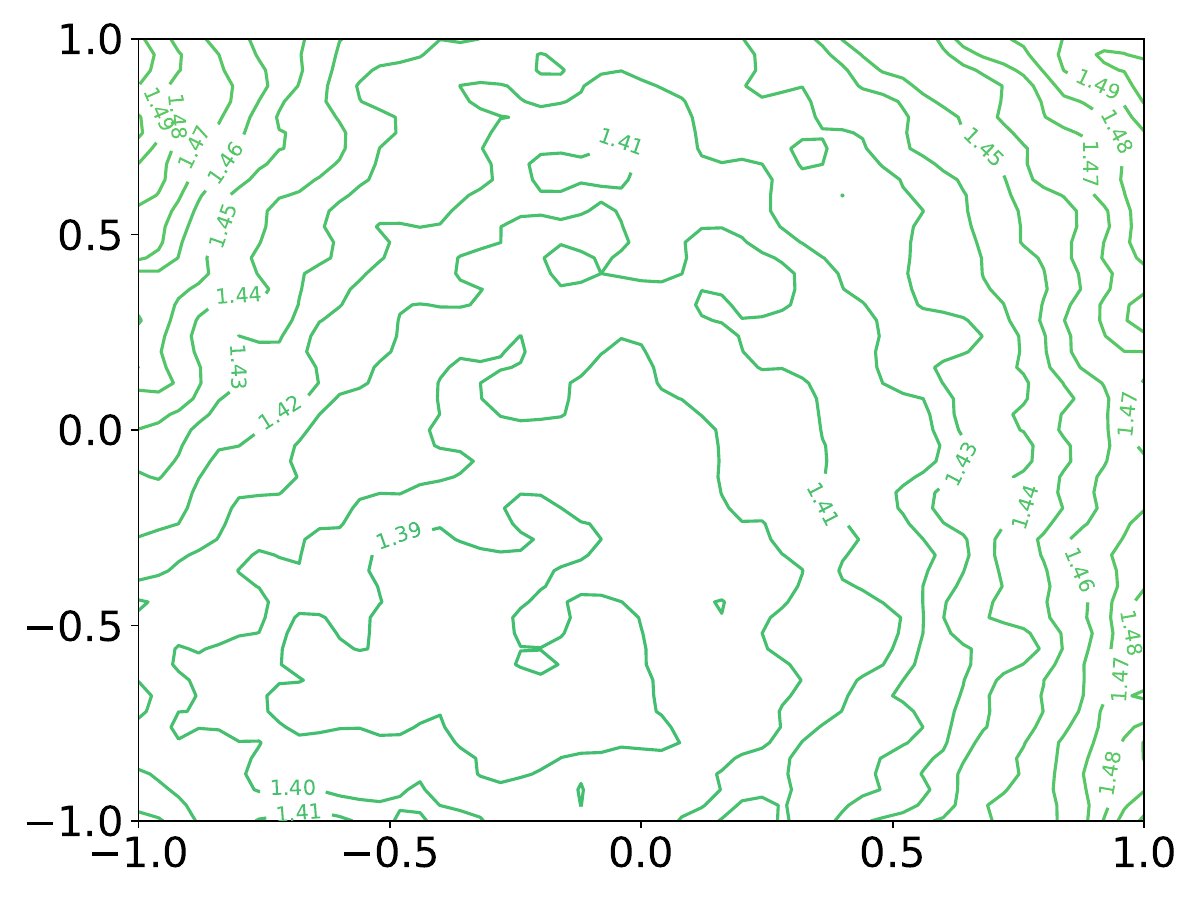}
\caption{CHARMS}
\end{subfigure}%
\hfill
\begin{subfigure}[t]{.49\linewidth}
\centering
\includegraphics[width=\linewidth]{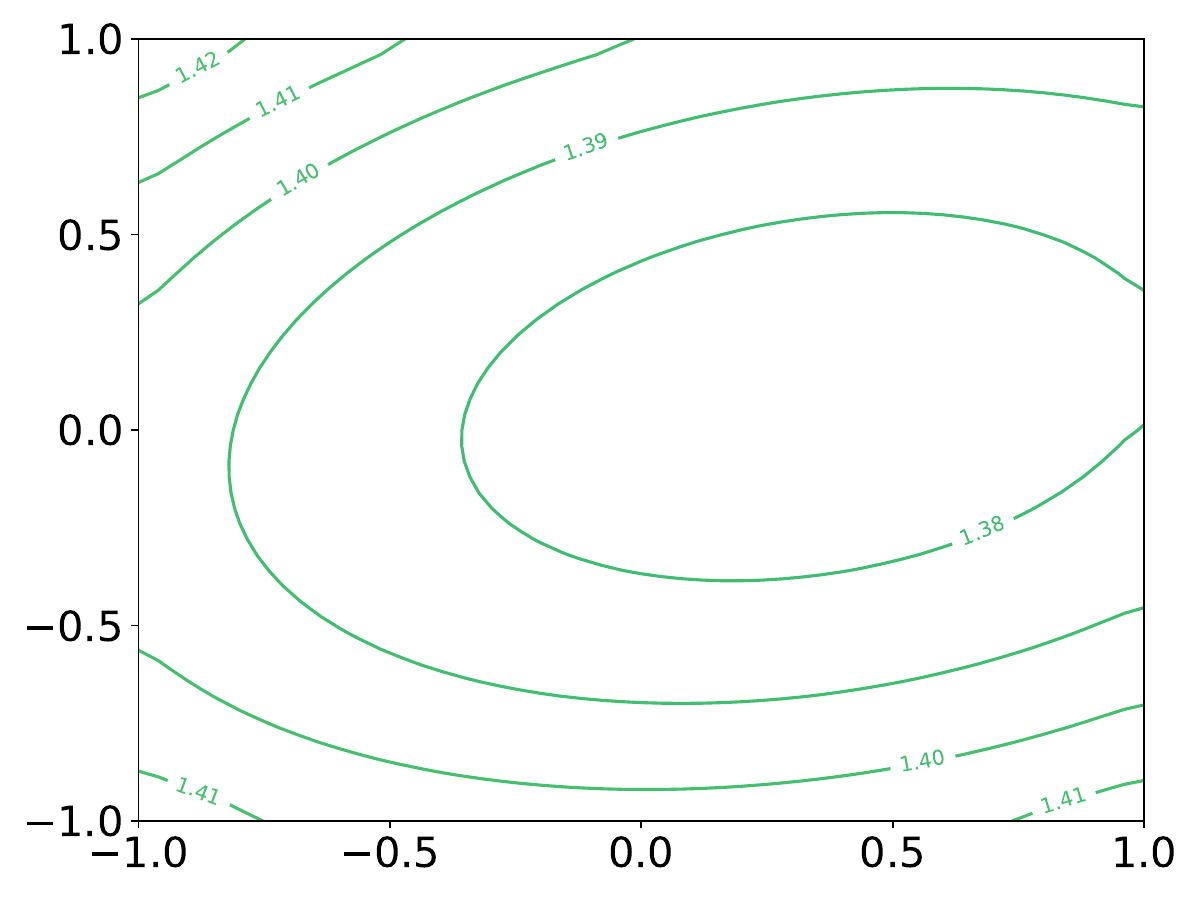}
\caption{Ours}
\end{subfigure}
\caption{Loss landscape on Adoption dataset.}
\label{fig:land}
\end{minipage}
\hfill
\begin{minipage}[t]{0.49\textwidth}
\centering
\begin{subfigure}[t]{.49\linewidth}
\centering
\includegraphics[width=\linewidth]{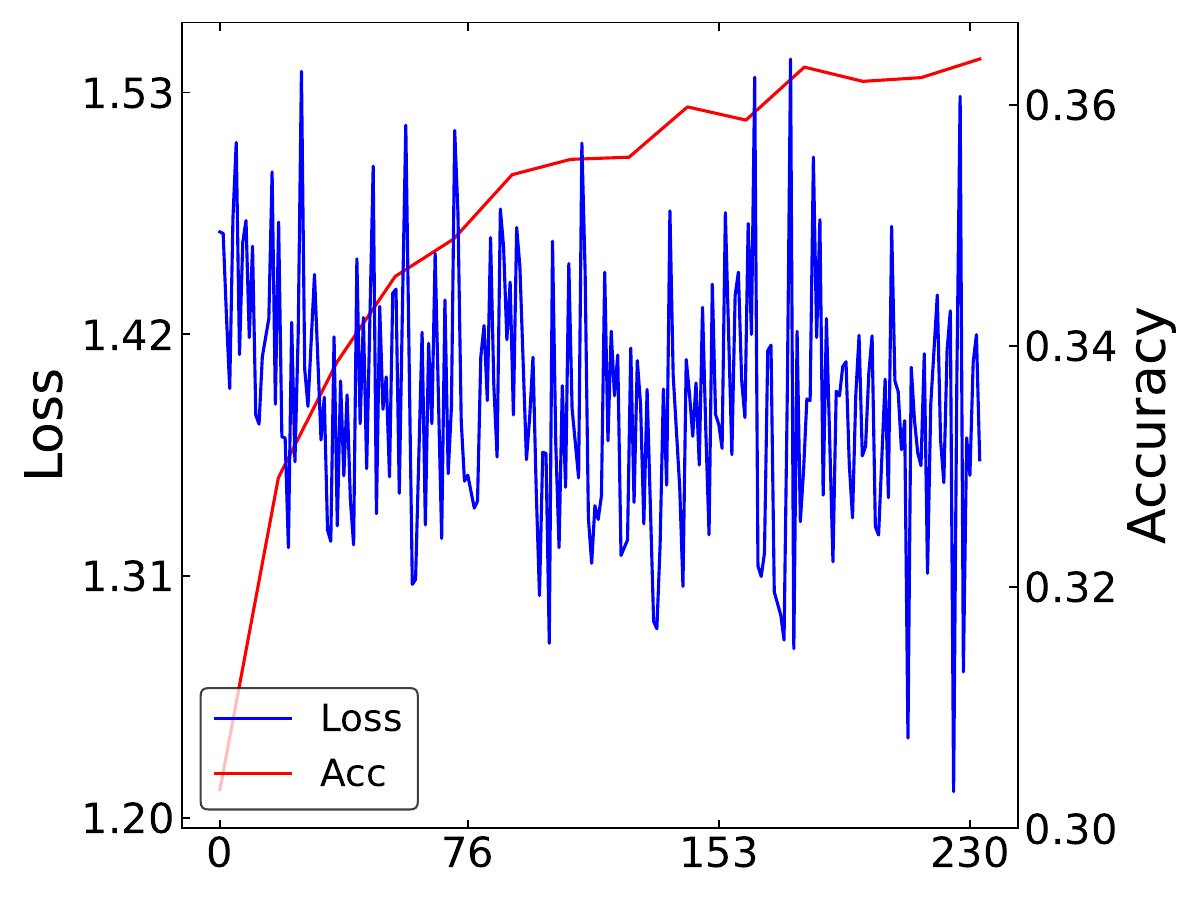}
\caption{CHARMS}
\end{subfigure}%
\hfill
\begin{subfigure}[t]{.49\linewidth}
\centering
\includegraphics[width=\linewidth]{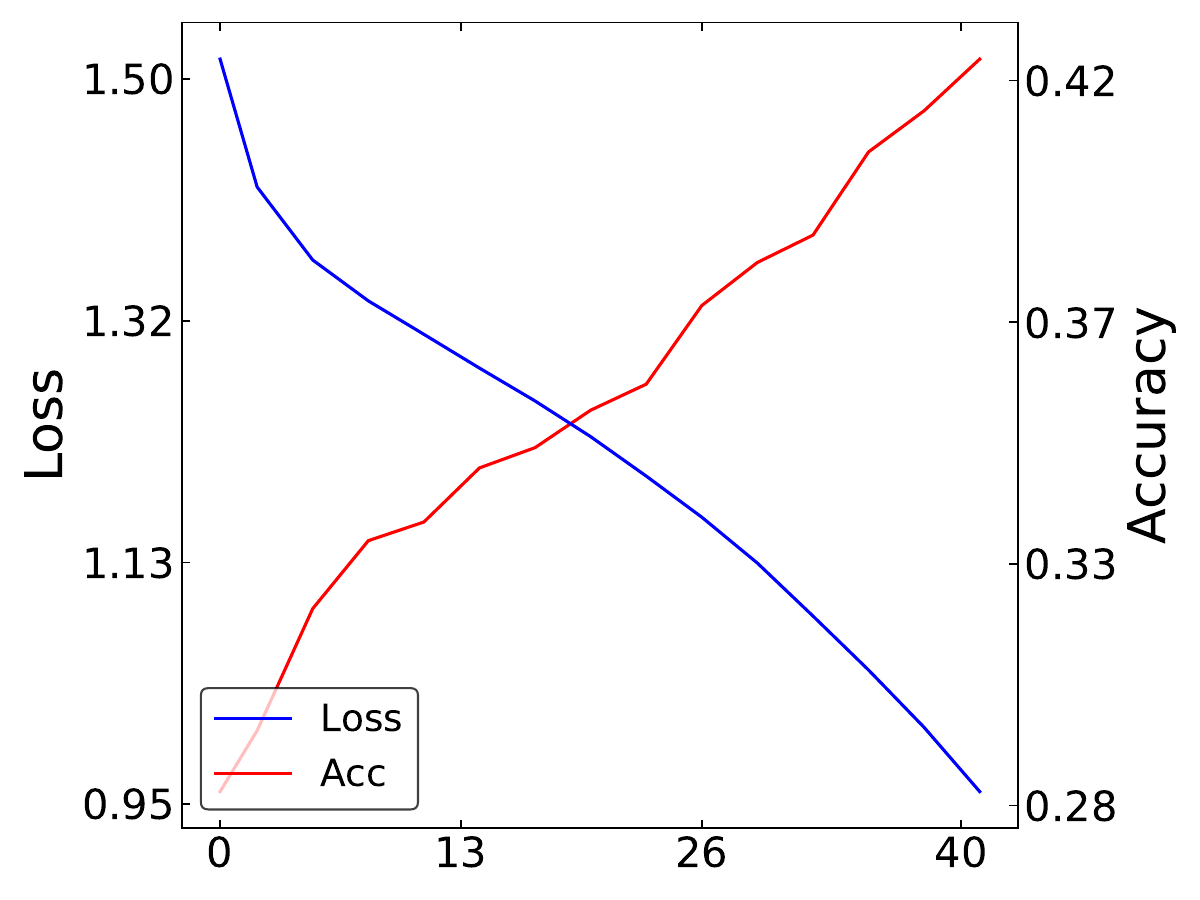}
\caption{Ours}
\end{subfigure}
\caption{Loss and accuracy on Adoption dataset.}
\label{fig:loss}
\end{minipage}
\end{figure*}

\subsection{Further Analysis}
\noindent\textbf{Comparison with SAM:} Sharpness aware minimization (SAM) \cite{Land:conf/iclr/ForetKMN21} improves the generalization ability by minimizing the loss value and sharpness. We apply SAM to unimodal learning to improve model quality and conduct experiments on classification tasks. As shown in Table \ref{tab:sam}, SKT that introduces cross-modal tabular information outperforms applying SAM to unimodal learning.

\noindent\textbf{Visualization Results:} In Figure \ref{fig:intro} and Figure \ref{fig:land}, we visualize the 2D loss landscapes of CHARMS and SKT on DVM and Adoption datasets, respectively, using DNN visualization \cite{Land:conf/nips/Li0TSG18,Land:conf/iclr/ForetKMN21}. We can find that the loss change of SKT is smaller than that of CHARMS. That is to say, the loss landscape of our proposed method SKT is flatter than that of CHARMS. 

\begin{figure}[t] 
\includegraphics[width=\linewidth]{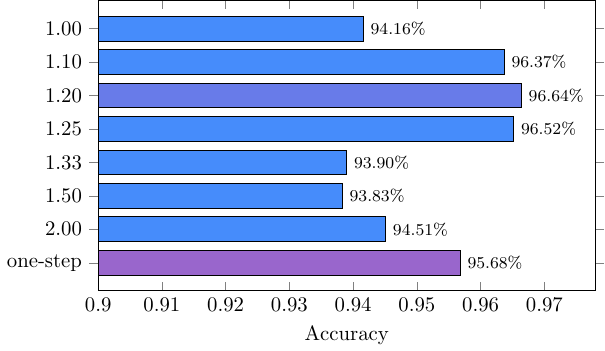}
\caption{Sensitivity analysis about $p$ and one-step asymmetric bargaining (one-step) analysis on DVM dataset.}
\label{fig:Nash}
\end{figure}

\noindent\textbf{Hyperparameter Sensitivity and Asymmetric Bargaining Analysis:} We study the influence of preference $p$ on DVM dataset. The accuracy with different $p \in \left [ 1.0, 2.0\right ]$ is reported in Figure \ref{fig:Nash}. An excessively large or small $p$ may weaken the auxiliary effect of the tabular modality. Although SKT shows some variation under different $p$, it remains highly effective and consistently outperforms the unimodal learning and shared head baseline.

We further investigate the effect of the asymmetric negotiation strategy on performance. In the two-step Nash Bargaining strategy, we introduce a modality preference value $p$ to control the direction of knowledge transfer. This naturally raises the following question: if we directly assign different preference values to tabular and image, can the transfer problem be addressed using one-step optimization, i.e., a Nash Bargaining Solution? Specifically, we can modify the Equation \ref{opt:balance} constraint to $-\log(\alpha_i)-\log((\S\boldsymbol{\alpha})_i)+\log(p_i)\le 0$, where tuning $p_i$ adjusts the modality preference. The entire process degenerates into a one-step optimization. By setting a larger $p_I$ and a smaller $p_T$ we can force the image modality to gain more benefit than the tabular modality. We conduct experiments on DVM dataset to analyze the effectiveness of this one-step asymmetric bargaining. In practice, we also apply hyperparameter search and early stopping to $p_I$ and $p_T$ to obtain the best results for one-step asymmetric bargaining. The results are shown as ``one-step'' in Figure \ref{fig:Nash}. We observe that one-step asymmetric bargaining also improves the unimodal model, but the gain remains smaller than that of the two-step Nash Bargaining strategy. This may be because it does not explicitly consider modality imbalance, and assigning different preference values affects all modalities, which can easily over-suppress or over-emphasize tabular information. More details are provided in the appendix.

\begin{table}[t]
\centering
\begin{tabular}{ccccc}
\toprule
\multirow{1}{*}{Method} & \multicolumn{1}{c}{\textbf{DVM}} & \multicolumn{1}{c}{\textbf{SUN}} & \multicolumn{1}{c}{\textbf{CelebA}} & \multicolumn{1}{c}{\textbf{Adoption}} \\
\midrule
Unimodal & 0.8743 & 0.8361 & 0.8164 & 0.3477 \\
\midrule
CLIP-LP & 0.7619 & 0.6918 & 0.7590 & 0.3047  \\
CLIP-FT & 0.8417 & 0.8333 & 0.8165 & 0.2935  \\
\midrule
SKT & \bf 0.9664 & \bf 0.8682 & \bf 0.8255 & \bf 0.3610  \\
\bottomrule
\end{tabular}
\caption{Comparison of our method SKT with the CLIP on the classification task.}
\label{tab:CLIP}
\end{table}

\noindent\textbf{Comparison with CLIP:} CLIP \cite{CLIP:conf/icml/RadfordKHRGASAM21}, pretrained on large-scale image-text pairs, learns cross-modal mapping from text to images. Previous works \cite{CLIP:conf/nips/Wang022,CHARMS:conf/icml/JiangYW00Z24} convert tabular data into text and attach a classification head to CLIP encoders, thereby formulating the task as language-to-vision cross-modal transfer based on CLIP. The results are reported in Table \ref{tab:CLIP}. CLIP-LP denotes fixing both encoders and training only the classification head, while CLIP-FT fine-tunes the entire CLIP network. We can find that CLIP performs poorly, and CLIP-FT generally outperforms CLIP-LP. This may be because tabular data remains difficult for CLIP to encode effectively even after being converted into text. In addition, CLIP treats all modalities equally, and recent studies also identify the modality gap \cite{Align:conf/nips/LiangZKYZ22,Align:conf/iclr/ZhangSY24}.

\begin{figure}[t] 
\includegraphics[width=0.96\linewidth]{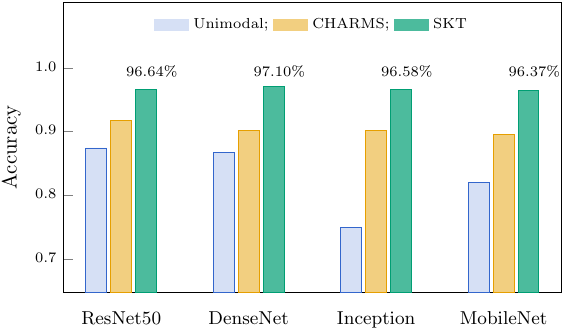}
\caption{Comparison of different network structures on DVM dataset. It is worth noting that our approach remains unaffected by backbone model.}
\label{fig:model}
\end{figure}

\noindent\textbf{Robustness of the Different Models:} To demonstrate the applicability and robustness of SKT, we conduct experiments using various network architectures beyond ResNet-50 on DVM dataset, including DenseNet-121 \cite{Densenet:conf/cvpr/HuangLMW17}, Inception-v1 \cite{Inception:conf/cvpr/SzegedyLJSRAEVR15}, and MobileNet-v2 \cite{MobileNet:conf/cvpr/SandlerHZZC18}. As shown in Figure \ref{fig:model}, our method consistently improves performance across different network architectures. Although the performance of unimodal learning varies significantly across different network architectures, the proposed method SKT achieves comparable and competitive results in all cases, highlighting its ability to transfer tabular knowledge to images effectively. 

\begin{table}[t]
\centering
\begin{tabular}{lccccc}
\toprule
\multirow{1}{*}{Method} & \multicolumn{1}{c}{Training Time} & \multicolumn{1}{c}{Accuracy} \\
\midrule
Unimodal & 2$h$47$m$ & 0.8743 \\
Baseline & 3$h$12$m$ & 0.8552 \\
CHARMS & 47$h$10$m$ & 0.9175 \\
\midrule
SKT & 3$h$31$m$ & \bf 0.9664 \\
\bottomrule
\end{tabular}
\caption{Results of training overhead on DVM dataset.}
\label{tab:time}
\end{table}

\noindent\textbf{Convergence and Overhead:} To validate the convergence of the method SKT, we conduct experiments on DVM and Adoption datasets. Figure \ref{fig:intro} and \ref{fig:loss} report the variation of the image loss during training. It is worth noting that we employ an early stopping strategy to prevent overfitting. It can be observed that, as training progresses, the image loss in SKT decreases rapidly and stably, whereas that in CHARMS fluctuates significantly and requires more iterations. Moreover, we compare the training cost of SKT with unimodal learning (Unimodal), shared head learning (Baseline), and CHARMS on DVM dataset. As shown in Table \ref{tab:time}, SKT achieves the best accuracy while maintaining competitive training time. During inference, our method introduces no additional parameters or post-processing strategies, thereby incurring no extra inference cost.

\section{Conclusion} \label{sec:conclusion}
In this paper, we propose Skewed Knowledge Transfer (SKT), a novel cross-modal transfer method to address tabular data is unavailable at inference. By considering modality imbalance and asymmetric cross-modal relationships, SKT enables directed transfer of tabular knowledge to facilitate image learning. Specifically, we first introduce a multimodal shared head, which enables heterogeneous models to benefit from cross-modal structure without incurring additional parameter overhead. We then design a two-step Nash Bargaining strategy to promote cross-modal transfer. In the first step, SKT seeks a balanced gradient weighting solution to mitigate modality imbalance. In the second step, a modality preference is introduced and only the image weight is adjusted, making the combined gradient more favorable to the image model. We further provide theoretical analysis of the Pareto improvement and convergence properties of SKT. Experiments on widely used tabular-image datasets show that the proposed SKT consistently improves image performance by leveraging tabular information.

\noindent\textbf{Limitations:} Our method relies on validation-based tuning to select the optimal preference value, which may limit flexibility. Beyond Nash Bargaining, other game-theoretic or gradient-integration methods may also be effective and require asymmetric designs. We leave these as future work.

\bibliography{aaai2027}


\clearpage
\appendix

\renewcommand{\thefigure}{\Roman{figure}}
\renewcommand{\thetable}{\Roman{table}}

\setcounter{figure}{0}
\setcounter{table}{0}
\setcounter{theorem}{0} 
\setcounter{lemma}{0}

\section{Proofs} \label{app:proof}
\begin{theorem}[Optimality condition]\label{proof:optimality}
Let some scalar $\mu>0$ and $L(\boldsymbol{\tilde{g}})=\sum_{r\in\{I,T\}}\log({u^r(\boldsymbol{\tilde{g}})})$. The optimal solution solution $\boldsymbol{\tilde{g}}^{*}$ to Equation \ref{obj:object} must satisfy:
\begin{align}
\nabla L(\boldsymbol{\tilde{g}}^{*})=\mu \boldsymbol{\tilde{g}}^{*}, \boldsymbol{\tilde{g}}^{*}=\G\boldsymbol{\alpha}. \label{proof:opt}
\end{align} 
\end{theorem}

\begin{proof} 
We consider the following constrained optimization problem:
\begin{equation}
\begin{aligned}
\max_{\boldsymbol{\tilde{g}}}& \sum_{r\in\{I,T\}}\log({u^r(\boldsymbol{\tilde{g}})}) \\
s.t.\;&\left \| \boldsymbol{\tilde{g}} \right \|^{2} \leq \epsilon^{2} \\ &u^r(\boldsymbol{\tilde{g}}) > 0, r\in \left \{ I,T \right \}. \label{proof:opt1} 
\end{aligned}
\end{equation}
Since the objective is strictly increasing with respect to positive scaling of $\boldsymbol{\tilde{g}}$, any optimal solution must satisfy $\left \| \boldsymbol{\tilde{g}}^{*} \right \|^{2} = \epsilon^2$.

Introduce a Lagrange multiplier $\mu \in \mathbb{R}$ for the equality constraint, and multipliers $\rho_r \leq 0$ for the inequality constraints. The Lagrangian is
\begin{equation}
\begin{aligned}
\FM\left (\boldsymbol{\tilde{g}},\mu,\rho_I,\rho_T \right )=&\sum_{r\in\{I,T\}}\log({u^r(\boldsymbol{\tilde{g}})})
+\mu \left (\left \| \boldsymbol{\tilde{g}} \right \|^{2} - \epsilon^{2} \right )\\ 
&-\sum_{r\in\{I,T\}} \rho_r\left (u^r(\boldsymbol{\tilde{g}}) \right ).
\end{aligned}
\end{equation}
Then, using the KKT conditions, we obtain 
\begin{equation}
\begin{aligned}
\sum_{r\in\{I,T\}} \frac{\g_r}{u^r(\boldsymbol{\tilde{g}^*})}+2 \mu \boldsymbol{\tilde{g}^*} - \sum_{r\in\{I,T\}} \rho_r \g_r=0. \label{proof:KKT}
\end{aligned}
\end{equation}
Due to $u^r(\boldsymbol{\tilde{g}^*})>0$, we have $\rho_r=0$. Thus, Equation \ref{proof:KKT} simplifies to 
\begin{equation}
\begin{aligned}
\sum_{r\in\{I,T\}} \frac{\g_r}{u^r(\boldsymbol{\tilde{g}^*})}=-2 \mu \boldsymbol{\tilde{g}^*}.
\end{aligned}
\end{equation}
Absorbing the constant factor and sign into a scalar, there exists $\mu$ such that
\begin{equation}
\begin{aligned}
\nabla L(\boldsymbol{\tilde{g}}^{*})=\sum_{r\in\{I,T\}} \frac{\g_r}{u^r(\boldsymbol{\tilde{g}^*})}=\mu \boldsymbol{\tilde{g}}^{*}. \label{proof:comple}
\end{aligned}
\end{equation}
Due to $u^r(\boldsymbol{\tilde{g}^*})>0$, let $\alpha_r=\frac{1}{u^r(\boldsymbol{\tilde{g}^*})}>0$. Substituting into Equation \ref{proof:comple} and normalizing the scalar factor, we obtain
\begin{equation}
\begin{aligned}
\boldsymbol{\tilde{g}}^{*}=\G\boldsymbol{\alpha}.
\end{aligned}
\end{equation}
This completes the proof.
\end{proof}

\begin{theorem}[Solution characterization]\label{proof:solution}
The solution to Equation \ref{thm:opt}, up to scaling, is $\tilde{g}^{*}=\G\boldsymbol{\alpha}$ where $\boldsymbol{\alpha}$ is the solution to
\begin{align}
\G^{\top}\G\boldsymbol{\alpha}=\frac{1}{\boldsymbol{\alpha}}, \label{proof:obj}
\end{align}
with the element-wise reciprocal $\tfrac{1}{\boldsymbol{\alpha}}$.
\end{theorem}

\begin{proof} 
We follow the same steps in Theorem 3.2 of \cite{Nash:conf/nips/0005YCF0J024}. From Theorem 2.3 (Optimality Condition), any Nash Bargaining Solution $\boldsymbol{\tilde{g}^*}$ must satisfy
\begin{equation}
\begin{aligned}
\boldsymbol{\tilde{g}}^{*}=\G\boldsymbol{\alpha}.
\end{aligned}
\end{equation}
Multiplying both sides with $\g_r$, yields
\begin{equation}
\begin{aligned}
{\g_r}^{\top}\boldsymbol{\tilde{g}}^{*}={\g_r}^{\top}\G\boldsymbol{\alpha}.
\end{aligned}
\end{equation}
Since $\alpha_r=\frac{1}{u^r(\boldsymbol{\tilde{g}^*})}>0$, we obtain
\begin{equation}
\begin{aligned}
\G^{\top}\G\boldsymbol{\alpha}=\frac{1}{\boldsymbol{\alpha}}.
\end{aligned}
\end{equation}
This completes the proof.
\end{proof}

\begin{lemma}[Boundedness]\label{proof:Boundedness}
If $\g_r$ is $\sigma$-bounded for $r \in \{ I, T\}$, then $\frac{1}{\sqrt{2}\sigma} \leq \left \| \alpha_T \right \|$ and $\frac{\sqrt{p}}{\sqrt{2}\sigma} \leq \left \| \alpha_I \right \|$, where $\sigma < \infty$ and preference $p>1$.
\end{lemma}

\begin{proof} 
For $\alpha_T$, because of $\G^{\top}\G\boldsymbol{\alpha}=\frac{1}{\boldsymbol{\alpha}}$, we have
\begin{equation}
\begin{aligned}
{\left \| \boldsymbol{\tilde{g}} \right \|}^{2}&=\boldsymbol{\tilde{g}}^{\top} \boldsymbol{\tilde{g}} \\
&=\left ( \G\boldsymbol{\alpha} \right ) ^{\top} \left ( \G\boldsymbol{\alpha} \right ) \\
&=\boldsymbol{\alpha}^{\top} \G^{\top} \G \boldsymbol{\alpha} \\
&=\boldsymbol{\alpha}^{\top} \frac{1}{\boldsymbol{\alpha}} \\
&=\sum_{r\in\{I,T\}} \alpha_r \cdot \frac{1}{\alpha_r}=2.
\end{aligned}
\end{equation}
By Cauchy-Schwarz inequality, we have
\begin{equation}
\begin{aligned}
\left \| \frac{1}{\alpha_T} \right \|=\left \|\g_T\G\boldsymbol{\alpha} \right \| \leq \left \|\G\boldsymbol{\alpha} \right \| \cdot \left \|\g_T \right \| \leq \sqrt{2}\sigma.
\end{aligned}
\end{equation}
Similarly, for $\alpha_I$, since $\G^{\top}\G\boldsymbol{\alpha}=\frac{p}{\boldsymbol{\alpha}}$, we have
\begin{equation}
\begin{aligned}
{\left \| \boldsymbol{\tilde{g}} \right \|}^{2}&=\boldsymbol{\alpha}^{\top} \frac{p}{\boldsymbol{\alpha}} \\
&=\sum_{r\in\{I,T\}} \alpha_r \cdot \frac{p}{\alpha_r}=2p.
\end{aligned}
\end{equation}
By Cauchy-Schwarz inequality, we have
\begin{equation}
\begin{aligned}
\left \| \frac{p}{\alpha_I} \right \|=\left \|\g_I\G\boldsymbol{\alpha} \right \| \leq \left \|\G\boldsymbol{\alpha} \right \| \cdot \left \|\g_I \right \| \leq \sqrt{2p}\sigma.
\end{aligned}
\end{equation}
Hence, we have 
\begin{equation}
\begin{aligned}
\frac{1}{\sqrt{2}\sigma} \leq \left \| \alpha_T \right \| \\
\frac{\sqrt{p}}{\sqrt{2}\sigma} \leq \left \| \alpha_I \right \|.
\end{aligned}
\end{equation}
This completes the proof.
\end{proof}

\begin{theorem}[Pareto improvement]\label{proof:Pareto}
Assume $\LM^r$ is differential and Lipschitz-smooth with constant $C>0$. If the learning rate at step $k$ is set to $\eta^{\left ( k\right )}\leq\frac{1}{C\alpha_{r}^{\left ( k\right )}}$, then the update ensures $\LM^r\left ( \Theta^{\left ( k+1\right )}\right ) \leq \LM^r \left ( \Theta^{\left ( k\right )}\right )$ for both players.
\end{theorem}

\begin{proof} 
$\LM^r$ is assumed to be Lipschitz-smooth with constant $C>0$, we can obtain $\left \| \nabla \LM^r(\Theta^{(k+1)})-\nabla \LM^r(\Theta^{(k)}) \right \| \leq C\left \| \Theta^{(k+1)}-\Theta^{(k)} \right \|$. Then, we employ Taylor's expansion of $\LM^r(\Theta^{(k+1)})$ around $\Theta^{(k)}$:
{\tiny
\begin{equation}
\begin{aligned}
&\LM^r(\Theta^{(k+1)})\\
&=\LM^r(\Theta^{(k)})+\int_{0}^{1} {\nabla \LM^r(\Theta^{(k)}+\tau (\Theta^{(k+1)}-\Theta^{(k)}))^{\top}(\Theta^{(k+1)}-\Theta^{(k)})\mathrm{d}\tau} \\
&=\LM^r(\Theta^{(k)})+\nabla \LM^r(\Theta^{(k)})^{\top}(\Theta^{(k+1)}-\Theta^{(k)}) \\
&+\int_{0}^{1} { (\nabla \LM^r(\Theta^{(k)}+\tau (\Theta^{(k+1)}-\Theta^{(k)}))-\nabla \LM^r(\Theta^{(k)}))^{\top}(\Theta^{(k+1)}-\Theta^{(k)}) \mathrm{d}\tau} \\
&\leq \LM^r(\Theta^{(k)})+\nabla \LM^r(\Theta^{(k)})^{\top}(\Theta^{(k+1)}-\Theta^{(k)}) \\
&+\int_{0}^{1} { \left \| \nabla \LM^r(\Theta^{(k)}+\tau (\Theta^{(k+1)}-\Theta^{(k)}))-\nabla \LM^r(\Theta^{(k)}) \right \| \cdot \left \|\Theta^{(k+1)}-\Theta^{(k)}\right \| \mathrm{d}\tau} \\
&\leq \LM^r(\Theta^{(k)})+\nabla \LM^r(\Theta^{(k)})^{\top}(\Theta^{(k+1)}-\Theta^{(k)}) \\
&+\int_{0}^{1} { C\left \| \tau (\Theta^{(k+1)}-\Theta^{(k)}) \right \| \cdot \left \|\Theta^{(k+1)}-\Theta^{(k)}\right \| \mathrm{d}\tau} \\
&= \LM^r(\Theta^{(k)})+\nabla \LM^r(\Theta^{(k)})^{\top}(\Theta^{(k+1)}-\Theta^{(k)})\\
&+C\left \| (\Theta^{(k+1)}-\Theta^{(k)}) \right \|^2 \int_{0}^{1} \tau\mathrm{d}\tau \\
&= \LM^r(\Theta^{(k)})+\nabla \LM^r(\Theta^{(k)})^{\top}(\Theta^{(k+1)}-\Theta^{(k)})+\frac{C}{2}\left \| \Theta^{(k+1)}-\Theta^{(k)} \right \|^2.
\end{aligned}
\end{equation}
}
For final combined gradient $\boldsymbol{\tilde{g}}$, since $\G^{\top}\G\boldsymbol{\alpha}=\frac{p}{\boldsymbol{\alpha}}$, we have 
\begin{equation}
\begin{aligned}
{\left \| \boldsymbol{\tilde{g}} \right \|}^{2}&=\boldsymbol{\tilde{g}}^{\top} \boldsymbol{\tilde{g}} \\
&=\left ( \G\boldsymbol{\alpha} \right ) ^{\top} \left ( \G\boldsymbol{\alpha} \right ) \\
&=\boldsymbol{\alpha}^{\top} \G^{\top} \G \boldsymbol{\alpha} \\
&=\boldsymbol{\alpha}^{\top} \frac{p}{\boldsymbol{\alpha}} \\
&=\sum_{r\in\{I,T\}} \alpha_r \cdot \frac{p}{\alpha_r}=2p.
\end{aligned}
\end{equation}
Then, with $\Theta^{(k+1)}=\Theta^{(k)}-\eta^{(k)}\boldsymbol{\tilde{g}}^{(k)}$ and $\nabla \LM^r(\Theta^{(k)})=\g_{r}^{k}$, we have
\begin{equation}
\begin{aligned}
&\LM^r(\Theta^{(k+1)})\\
&\leq \LM^r(\Theta^{(k)})-\eta^{(k)}(\g_{r}^{k})^{\top}\boldsymbol{\tilde{g}}^{(k)}+\frac{C}{2}\left \| \eta^{(k)}\boldsymbol{\tilde{g}} \right \|^2 \\
&=\LM^r(\Theta^{(k)})-\eta^{(k)} \cdot \frac{p}{\alpha_{r}^{(k)            }}+Cp(\eta^{(k)})^2 \\
&\leq \LM^r(\Theta^{(k)})+\eta^{(k)}p \cdot (C\frac{1}{C\alpha_{r}^{(k)}}-\frac{1}{\alpha_{r}^{(k)}}) \\
&\leq \LM^r(\Theta^{(k)}).
\end{aligned}
\end{equation}
This completes the proof.
\end{proof}
The final solution satisfies Pareto improvement. With $\alpha_T$ fixed in the second bargaining, the resulting combined gradient $\boldsymbol{\tilde{g}}$ update is effectively biased toward the images.

\begin{theorem}[Convergence]\label{proof:Convergence}
Assume unimodal loss $\LM^r\left ( \Theta^{\left ( k\right )}\right )$ is monotonically decreasing and bounded below. Then the multimodal loss $\LM\left ( \Theta\right )$ converges to $\LM\left ( \Theta^*\right )$ and $\Theta^*$ is the stationary point of $\LM\left ( \Theta\right )$. 
\end{theorem}

\begin{proof} 
Let $\left \|\g \right \|\leq\sigma=1.0$ to ensure stability during optimization. Since the lower bound of $\alpha_T$ is smaller than that of $\alpha_I$, so learning rate $\eta^{\left ( k\right )}\leq\frac{1}{C\alpha_{r}^{\left ( k\right )}}\leq \frac{\sqrt{2}\sigma}{C}\leq \frac{\sqrt{2}}{C}< \frac{2}{C}$.
By Theorem \ref{proof:Pareto}, for the multimodal loss $\LM$, we have
\begin{equation}
\begin{aligned}
&\LM(\Theta^{(k+1)})\\
&\leq \LM(\Theta^{(k)})-\eta^{(k)}(\nabla \LM(\Theta^{(k)}))^{\top}\boldsymbol{\tilde{g}}^{(k)}+\frac{C}{2}\left \| \eta^{(k)}\boldsymbol{\tilde{g}} \right \|^2 \\
&= \LM(\Theta^{(k)})-\eta^{(k)}{\left \| \boldsymbol{\tilde{g}} \right \|}^{2}+\frac{C}{2}\left \| \eta^{(k)}\boldsymbol{\tilde{g}} \right \|^2 \\
&=LM(\Theta^{(k)})-\eta^{(k)}{\left \| \boldsymbol{\tilde{g}} \right \|}^{2}(\frac{C}{2} \eta^{(k)} -1) \\
&<\LM(\Theta^{(k)}).
\end{aligned}
\end{equation}

Since $\LM(\Theta^{(k)})$ is monotonically decreasing and bounded below by 0, it must converge to some limit value $\LM\left ( \Theta^*\right )$. This implies that as $k\rightarrow \infty$, the improvement in loss $\eta^{(k)}{\left \| \boldsymbol{\tilde{g}}^{k} \right \|}^{2}\rightarrow 0$. Hence, we have $\boldsymbol{\tilde{g}}=0$ at $\Theta^*$,  indicating that $\Theta^*$ is the stationary point of $\LM(\Theta)$ and the modality gradients are linearly dependent. Any small movement from $\Theta^*$ will improve another player only at the expense of the other, therefore $\Theta^*$ is the Pareto stationary point.

This completes the proof.
\end{proof}

\section{Notation Definition}
We summarize the notation definition we used in this paper in Table \ref{tab:notation}.

\begin{table}[t]
\centering
\begin{tabular}{l|l}
\hline
\textbf{Notation} & \textbf{Description} \\
\hline
$N$                    & The number of training data. \\
$Y$                    & The number of category labels. \\
$\x_i^I / \x_i^T$      & Image / Tabular data point. \\
$\DM$          & Training set. \\
$y_i$          & Category label of $i$-th data. \\
$r$        & $r$-modality. \\
$\phi^r$               & Encoder of $r$ modality. \\
$\theta^r$             & Parameters of $\phi^r$. \\
$\z^r$                  & Feature of $\x^r$. \\
$\psi$               & Multimodal shared head. \\
$\Theta$          & Parameters of $\psi$. \\
$f^r$             & Model of $r$ modality. \\
$f^r(\x^r)$                  & Prediction of $\x^r$. \\
$\ell(\cdot)$          & Task loss (Cross-entropy or MSE). \\
$\LM$ & Multimodal loss. \\
$\LM^r$   & Unimodal loss. \\
$\g_r$                    & $r$ modality gradient for shared head. \\
$\G$                    & Gradient matrix. \\
$\alpha_r$               & $r$ modality gradient weight. \\
$\boldsymbol{\alpha}$               & Gradient weights matrix. \\
$u^r$                & $r$ modality utility function. \\
$\boldsymbol{\tilde{g}}$  & Multimodal combined gradient. \\
$\B_{\epsilon}$              & Ball of radius $\epsilon$. \\
$\S$                  & Gram matrix of gradients. \\
$\gamma(\cdot)$       & First-order regularization. \\
$p$                   & Preference constant. \\
$k$                   & Iteration step. \\
$\sigma$              & Bound of gradient norm. \\
$C$                   & Lipschitz constant. \\
$\eta$                & Learning rate. \\
\bottomrule
\end{tabular}
\caption{Notation Definition}
\label{tab:notation}
\end{table}

\begin{table*}[ht]
\centering
\renewcommand\arraystretch{1.3}
\resizebox{\textwidth}{!}{
\begin{tabular}{lc c c c c c}
\toprule
\multirow{1}{*}{\textbf{Method}} & \multicolumn{1}{c}{\textbf{DVM} $\uparrow$} & \multicolumn{1}{c}{\textbf{SUN} $\uparrow$} & \multicolumn{1}{c}{\textbf{CelebA} $\uparrow$} & \multicolumn{1}{c}{\textbf{Adoption} $\uparrow$} & \multicolumn{1}{c}{\textbf{Pawpularity} $\downarrow$} & \multicolumn{1}{c}{\textbf{Avito} $\downarrow$} \\
\midrule
Unimodal & 0.8743 \std{.0014} & 0.8361 \std{.0003} & 0.8146 \std{.0005} & 0.3477 \std{.0053} & 18.6150 \std{.0072} & 0.2512 \std{.0011} \\
\midrule
KD     & 0.8390 \std{.0076} & 0.8382 \std{.0063} & 0.8118 \std{.0046} & 0.3532 \std{.0035} & 19.0683 \std{1.7642} & 0.2499 \std{.0015} \\
MFH    &   -    & 0.8312 \std{.0022} & 0.7507 \std{.0034} & 0.3041 \std{.0027} & 43.1455 \std{2.0843} & 0.2873 \std{.0047} \\
FMR    & 0.8427 \std{.0151} & 0.8347 \std{.0119} & 0.8003 \std{.0143} & 0.3526 \std{.0088} & 19.3517 \std{1.5837} & 0.2937 \std{.0084} \\
MMCL   & 0.8203 \std{.0040} & 0.8431 \std{.0012} & 0.8041 \std{.0017} & 0.2981 \std{.0026} &    -    &    -   \\
CHARMS & \underline{0.9175} \std{.0052} & \underline{0.8661} \std{.0032} & \underline{0.8220} \std{.0022} & \underline{0.3603} \std{.0037} & \underline{18.4314} \std{.7427} & \underline{0.2495} \std{.0025} \\
\midrule
Baseline    & {0.8552} \std{.0049} & {0.8438} \std{.0037} & {0.8006} \std{.0056} & {0.3514} \std{.0031} & {19.1438} \std{1.2418} & {0.2539} \std{.0038} \\ \midrule
SKT    & {\bf0.9664  \std{0.0012}} & {\bf0.8682  \std{0.0011}} & {\bf0.8273  \std{0.0014}} & {\bf0.3610  \std{0.0043}} & {\bf18.2809  \std{0.0578}} & {\bf0.2437  \std{0.0007}} \\
\bottomrule
\end{tabular}
}
\caption{Comparisons with baseline methods on DVM, SUNAttribute, CelebA, Adoption, Pawpularity, and Avito datasets. The best and second-best results are highlighted in bold and underline, respectively.}
\label{proof:main-exp}
\end{table*}

\section{Experiments}
\subsection{Datasets}
The datasets used in our experiments are Data Visual Marketing (DVM), SUNAttribute, CelebA, PetFinder-adoption (Adoption), PetFinder-pawpularity (Pawpularity) and Avito.

\noindent{\textbf{DVM.}} The DVM \cite{DVM:conf/bigdataconf/HuangCLYO22} dataset  was created from 335,562 used car advertisements sourced from a car application. The DVM dataset contains 1,451,784 car images paired with corresponding tabular data, including sales and technical specifications. It consists of two data modalities: image data and tabular data, forming a total of 176,451 paired samples across 129 vehicle categories. All images are resized to a resolution of $128 \times 128$. The tabular data contains 17 features, including 4 categorical features and 13 continuous features. Following previous work \cite{CHARMS:conf/icml/JiangYW00Z24}, we preprocessed the data to obtain 70,580 training pairs, 17,645 validation pairs, and 88,226 test pairs. Car models with less than 700 samples were removed, resulting in 129 target classes.

\noindent{\textbf{SUNAttribute.}} The SUNAttribute \cite{SUN:journals/ijcv/PattersonXSH14} dataset contains 14,340 tabular-image pairs, constructed by sampling 20 annotated scenes from each of 717 SUN categories. The task is to predict whether a scene represents an open space. All images are resized to a resolution of $225 \times 225$. The corresponding tabular data comprises 101 categorical features. We follow the preprocessing procedure described in \cite{CHARMS:conf/icml/JiangYW00Z24}. The dataset is split into training, validation, and test sets in an 8:1:1 ratio.

\noindent{\textbf{CelebA.}} The CelebA \cite{CelebA:conf/iccv/LiuLWT15} dataset comprises facial images of 10,177 celebrities, forming a total of 202,599 tabular-image pairs. In this dataset, the "Attractive" label is used for classification. All images are uniformly resized to a resolution of 225×225 to ensure consistency. The tabular data consists of 39 categorical features, which are manually curated attributes derived from the facial images (e.g., Big\_Nose, etc.). Data preprocessing is performed according to the procedures described in \cite{CHARMS:conf/icml/JiangYW00Z24}. The dataset is also divided into training, validation, and test sets in an 8:1:1 ratio.

\noindent{\textbf{PetFinder-adoption.}} The Adoption \cite{CHARMS:conf/icml/JiangYW00Z24} dataset comes from a Kaggle competition where the task is to predict the speed at which a pet is adopted, which is a five-class classification task. There are total of 10 numerical variables and 14 categorical variables in this dataset. Tabular data contains information about the pet such as the type and vaccination status.

\noindent{\textbf{PetFinder-pawpularity.}} The Pawpularity \cite{CHARMS:conf/icml/JiangYW00Z24} dataset also comes from a Kaggle competition where the task was to predict the popularity of a pet based on that pet’s profile and photo, which is a regression task. Each pet photo is labeled with the value of 1 (Yes) or 0 (No) for each of the features. There are 12 categorical variables in tabular data.

\noindent{\textbf{Avito.}} The Avito \cite{CHARMS:conf/icml/JiangYW00Z24} dataset comes from Russia’s largest classified advertisements website, Avito. This dataset aims to predict demand for an online advertisement based on its full description, contextual information, and historical demand for similar advertisements in comparable contexts. It’s a regression task. There are a total of 2 numerical variables and 11 categorical variables in tabular data.

\subsection{Baseline Details}
Beyond unimodal baselines, we compare against several knowledge transfer methods, including KD \cite{KD:journals/corr/HintonVD15}, MFH \cite{MFH:conf/iclr/XueGRZ23}, FMR \cite{FMR:conf/aaai/YangZFJZ17}, MMCL \cite{MMCL:conf/cvpr/HagerMR23}, and CHARMS \cite{CHARMS:conf/icml/JiangYW00Z24}. Our overall training protocol follows MMCL. For fair comparison, we apply the same backbone, initialization, and data preprocessing to KD, MFH, FMR, and MMCL. However, since CHARMS requires a specific tabular backbone and specialized data processing, we strictly follow the training settings from its original paper.

\noindent{\textbf{KD \cite{KD:journals/corr/HintonVD15}:}} We search the temperatures in \{1.0, 2.0, 4.0, 6.0, 8.0\} and distillation loss weight in \{0.2, 0.4, 0.6, 0.8\}.

\noindent{\textbf{MFH \cite{MFH:conf/iclr/XueGRZ23}:}} We set modality general decisive information according to the feature ranking algorithm. The number of features is $50\%$ that for all features.

\noindent{\textbf{FMR \cite{FMR:conf/aaai/YangZFJZ17}:}} We set $10\%$ of the fixed features to be knocked down in each epoch and search the knockdown number in \{0.1, 0.2, 0.3, 0.4\}.

\noindent{\textbf{MMCL \cite{MMCL:conf/cvpr/HagerMR23}:}} We follow the hyperparameter settings from its original paper.

\noindent{\textbf{CHARMS \cite{CHARMS:conf/icml/JiangYW00Z24}:}} We follow the hyperparameter settings from its original paper.

\subsection{More Experiment Results}
We run the experiments for our method five times with different initial random seeds to remove the randomness. Then, we report the average performance with $std$ on all datasets in the Table \ref{proof:main-exp}.

\subsection{More Analysis of Asymmetric Bargaining}
When different preference values are directly assigned to the tabular and image modalities, the lower bound of $\boldsymbol{\alpha}$ is correspondingly influenced by the preference values. Assume that the preference values for the tabular and image modalities are denoted by $\alpha_T$ and $\alpha_I$, respectively. According to Lemma \ref{proof:Boundedness}, the preference-based lower bound of $\boldsymbol{\alpha}$ is given by:
\begin{equation}
\begin{aligned}
\frac{p_T}{\sqrt{p_T+p_I}\sigma} \leq \left \| \alpha_T \right \| \\
\frac{p_I}{\sqrt{p_T+p_I}\sigma} \leq \left \| \alpha_I \right \|.
\end{aligned}
\end{equation}
Therefore, the lower bounds of $\alpha_T$ and $\alpha_I$ are influenced not only by their own preferences but also by cross-modal preferences. It may easily over-suppress or over-emphasize tabular information.

\end{document}